\documentclass[conference]{IEEEtran}
\IEEEoverridecommandlockouts
\usepackage{cite}
\usepackage{amsmath,amssymb,amsfonts}
\usepackage{algorithmic}
\usepackage{algorithm}
\usepackage{graphicx}
\usepackage[caption=false,font=footnotesize]{subfig}
\usepackage{textcomp}
\usepackage[dvipsnames,svgnames,table]{xcolor}
\usepackage{booktabs}
\usepackage{makecell}
\newcommand{\methodname}{ATONE}
\newcommand{\method}{\textbf{\texttt{\methodname}}}
\newcommand{\gray}[1]{\textcolor{gray}{#1}}
\newcommand{\venue}[1]{{\scriptsize\textcolor{gray!70!black}{[#1]}}}
\colorlet{bestcell}{red!13}
\colorlet{secondcell}{blue!10}
\newcommand{\best}[1]{\cellcolor{bestcell}\textbf{#1}}
\newcommand{\second}[1]{\cellcolor{secondcell}\textbf{#1}}
\DeclareRobustCommand{\swatch}[1]{{\setlength{\fboxsep}{0pt}\setlength{\fboxrule}{0.4pt}\raisebox{-0.2ex}{\fcolorbox{gray!80}{#1}{\rule{0pt}{1.8ex}\hspace{5ex}}}}}
\newcommand{\band}[1]{\gray{\textit{#1}}}
\definecolor{Band}{HTML}{8080FF}
\definecolor{hdrA}{HTML}{8A1F3D}
\definecolor{hdrB}{HTML}{16479B}
\definecolor{hdrC}{HTML}{15683A}
\definecolor{hdrD}{HTML}{4D4D4D}
\newcommand{\hd}[2]{\cellcolor{#1}\textcolor{white}{\textbf{#2}}}
\definecolor{oursname}{HTML}{FBE5D0}

\newcommand{\std}[1]{{\scriptsize\textcolor{ForestGreen}{($\pm$#1)}}}
\usepackage{tcolorbox}
\usepackage{ntheorem}
\theoremseparator{:}

\newtheorem{assumption}{Assumption}
\newtheorem{lemma}{Lemma}
\newtheorem{proposition}{Proposition}
\newtheorem{theorem}{Theorem}

\theoremstyle{nonumberplain}
\theorembodyfont{\normalfont}
\newtheorem{proof}{Proof}
\usepackage{hyperref}
\usepackage{balance}

\def\BibTeX{{\rm B\kern-.05em{\sc i\kern-.025em b}\kern-.08em
    T\kern-.1667em\lower.7ex\hbox{E}\kern-.125emX}}
\begin{document}
\bstctlcite{IEEEexample:BSTcontrol}

\title{Initiating Removal from Operational Symptoms in Federated Vehicular Networks}

\author{
\IEEEauthorblockN{Zihao Ding\textsuperscript{1}, Jun Huang\textsuperscript{1}, and Qiang Duan\textsuperscript{2}}
\IEEEauthorblockA{\textsuperscript{1}Department of Electrical Engineering and Computer Science, South Dakota State University, Brookings, SD 57007, USA\\
\textsuperscript{2}Information Sciences and Technology Department, Pennsylvania State University, Abington, PA 19001, USA\\
Email: zihao.ding@jacks.sdstate.edu, jun.huang@sdstate.edu, qduan@psu.edu}
}

\maketitle

\begin{abstract}
Connected vehicles train a driving model through federated learning (FL) and drive with it, so a faulty contribution often surfaces only as a delayed symptom on the road. Federated unlearning (FU) assumes that a removal request names what to forget. A symptom names no vehicle or round, and aggregation has already spread its cause across the fleet. We propose \underline{A}ggregation-\underline{T}ime \underline{O}perational-evidence \underline{N}eutralization of \underline{E}ffects (\method{}), which makes unlearning a term of the FL aggregation rule. \method{} projects everyday driving errors and reported symptoms onto a subspace spanned by recent global increments. It propagates the surviving effect of every contribution with curvature estimated from these increments and scores all contributions in one backward pass. A two-tier correction suppresses suspected effects in every round and eliminates confirmed ones after a one-round probe. We prove that backward scoring matches forward scoring and that the elimination is exact when training is flat outside the subspace.
On three driving datasets, \method{} leaves 43\% to 61\% less residual harm than the strongest baseline in eight of nine settings while staying within 0.36 points of full retraining in utility.
\end{abstract}


\section{Introduction}
\label{sec:intro}

Mobile autonomous systems keep learning after deployment, so the knowledge they hold at the edge changes over their operational lifetime \cite{Li2026COMST}. Connected vehicles follow this pattern by training a shared driving model through federated learning (FL), coordinated by a roadside unit~\cite{Jin2026TVT}. The same model drives the fleet. A faulty contribution, whether mislabeled in fog or crafted by an attacker~\cite{Bagdasaryan2020AISTATS}, therefore reaches every vehicle through aggregation. Vehicles that leave the range of the roadside unit drive with an older copy until they reconnect, so the fault outlives the round that added it.

Federated unlearning (FU) removes the effect of an aggregated contribution~\cite{Cao2023SP,Pan2025AAAI,Yuan2025TIFS}. These methods share one assumption. A removal request from a departing client or from a data owner exercising the right to erasure~\cite{Ding2026ICDCS} names what to forget, so the procedure starts with a known target. A vehicle fleet receives no such request. Harm surfaces instead as a symptom, such as a near miss, many rounds later at a vehicle that never made the faulty contribution. A symptom says that the fleet has learned something wrong, but not which piece of learning is at fault. By then, aggregation has copied the contribution into every model downloaded since, so removing it at the source vehicle alone is not enough. A compromised vehicle may fabricate a symptom to frame an honest contributor.

Defenses against poisoning do not close this gap. Robust aggregation~\cite{Cao2021NDSS} and anomaly detection~\cite{Zhang2024TIFS} inspect updates when they are submitted. A contribution collected in fog, however, looks normal at that moment, and its contributor need not be malicious. Durable backdoors are also built to survive benign training long after the attack stops~\cite{Song2026INFOCOM}. Post-hoc tracing~\cite{Jia2025NeurIPS} names suspected clients, but it needs a labeled misclassified input and stored checkpoints, and it leaves removal to another method. It also scores the influence that a contribution once had, so it flags old effects that training has long since washed out.

The FL loop, however, carries the signals needed to infer the removal target. Every vehicle already differentiates its loss, so one more backward pass on recent driving data shows where the deployed model errs without a labeled failure case. The recent global increments also reveal how training shrinks an old effect. Influence functions need Hessian-vector products for this and are fragile in deep networks~\cite{Bae2022NeurIPS}, whereas the increments come free at the coordinator. Together, the two signals measure the harm a contribution still causes, which is what deserves removal. That harm can only persist along flat directions that training does not wash out, where it stays exactly as submitted and can therefore be subtracted at once.

Motivated by these insights, we design an in-loop unlearning method for vehicular FL. Aggregation-Time Operational-evidence Neutralization of Effects (\method{}) makes unlearning a term of the FL aggregation rule rather than a procedure run after an incident. \method{} projects everyday driving errors and reported symptoms onto the subspace of recent increments. It propagates the surviving effect of every contribution with curvature estimated from those increments and scores all contributions in one backward pass. A two-tier correction then suppresses suspected effects in every round and eliminates a confirmed one after a one-round probe, which also filters fabricated symptoms. The probe measures the utility loss of each elimination instead of estimating it. The coordinator stores no past model or update, and its only long-term state is one small curvature operator per window, regardless of the fleet size and the number of contributions it has aggregated.

Our contributions are summarized as follows:
\begin{itemize}
\item We formulate FU for a fleet whose removal target is inferred from delayed symptoms, as minimizing residual harm under per-round utility and total cost budgets.
\item We propose \method{}, which aggregates operational evidence, scores residuals in one backward pass, and corrects the model in two tiers inside the aggregation rule at the roadside unit, without storing any past checkpoint.
\item We prove that the propagation operator is the training curvature inside the evidence subspace, that backward and forward scoring agree, and that the two-part elimination returns the model trained without the removal target.
\item On three datasets, \method{} leaves 43\% to 61\% less residual harm than the strongest baseline in eight of nine settings and keeps utility within 0.36 points of retraining.
\end{itemize}

\section{System Model and Problem Formulation}
\label{sec:system}

\begin{figure}[t]
  \centering
  \includegraphics[width=0.95\columnwidth]{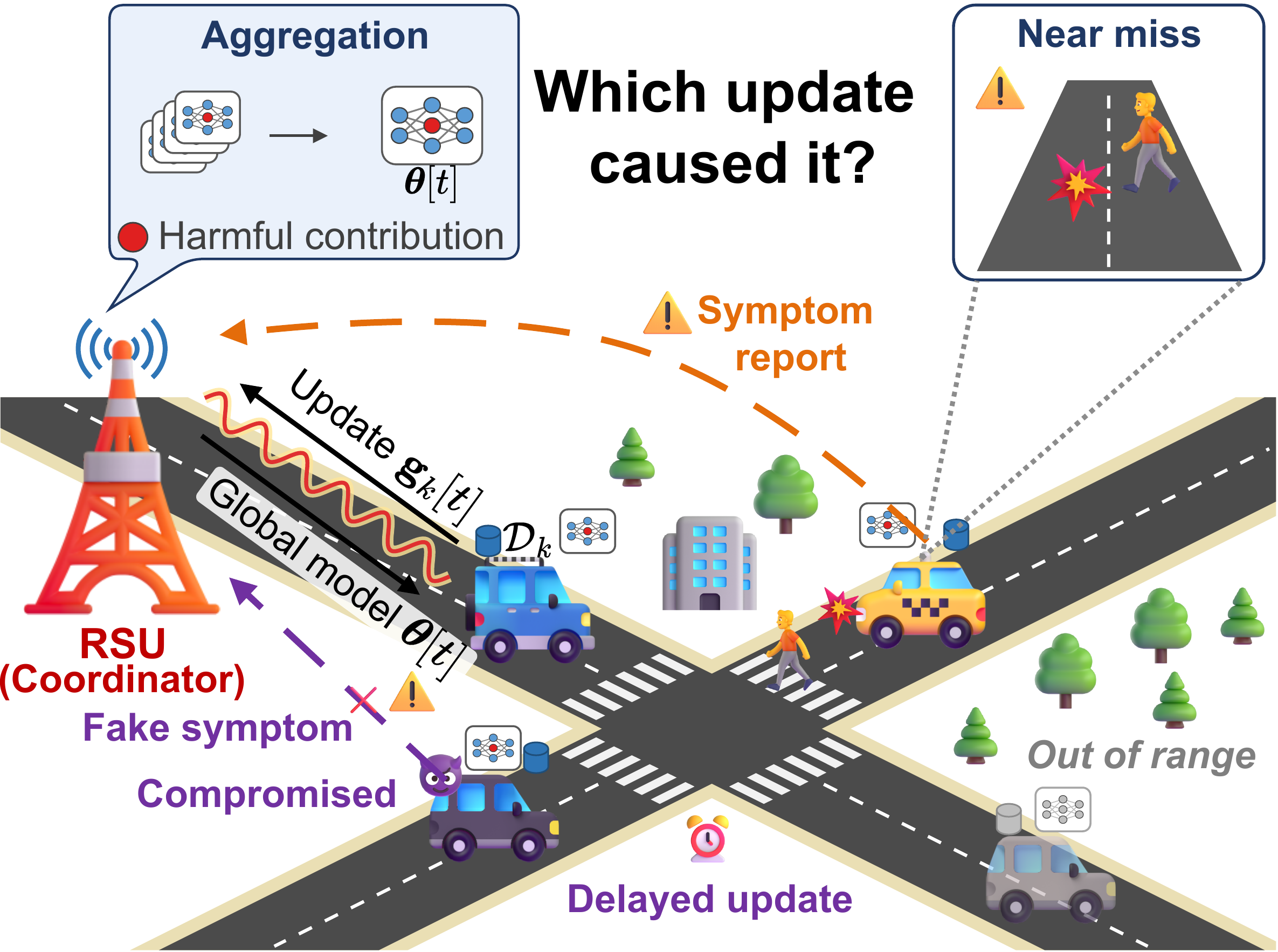}
  \caption{The system model of symptom-triggered unlearning.}
  \label{fig:system}
\end{figure}

\begin{figure*}[t]
  \centering
  \begin{minipage}[t]{0.48\textwidth}
    \centering
    \includegraphics[width=\linewidth]{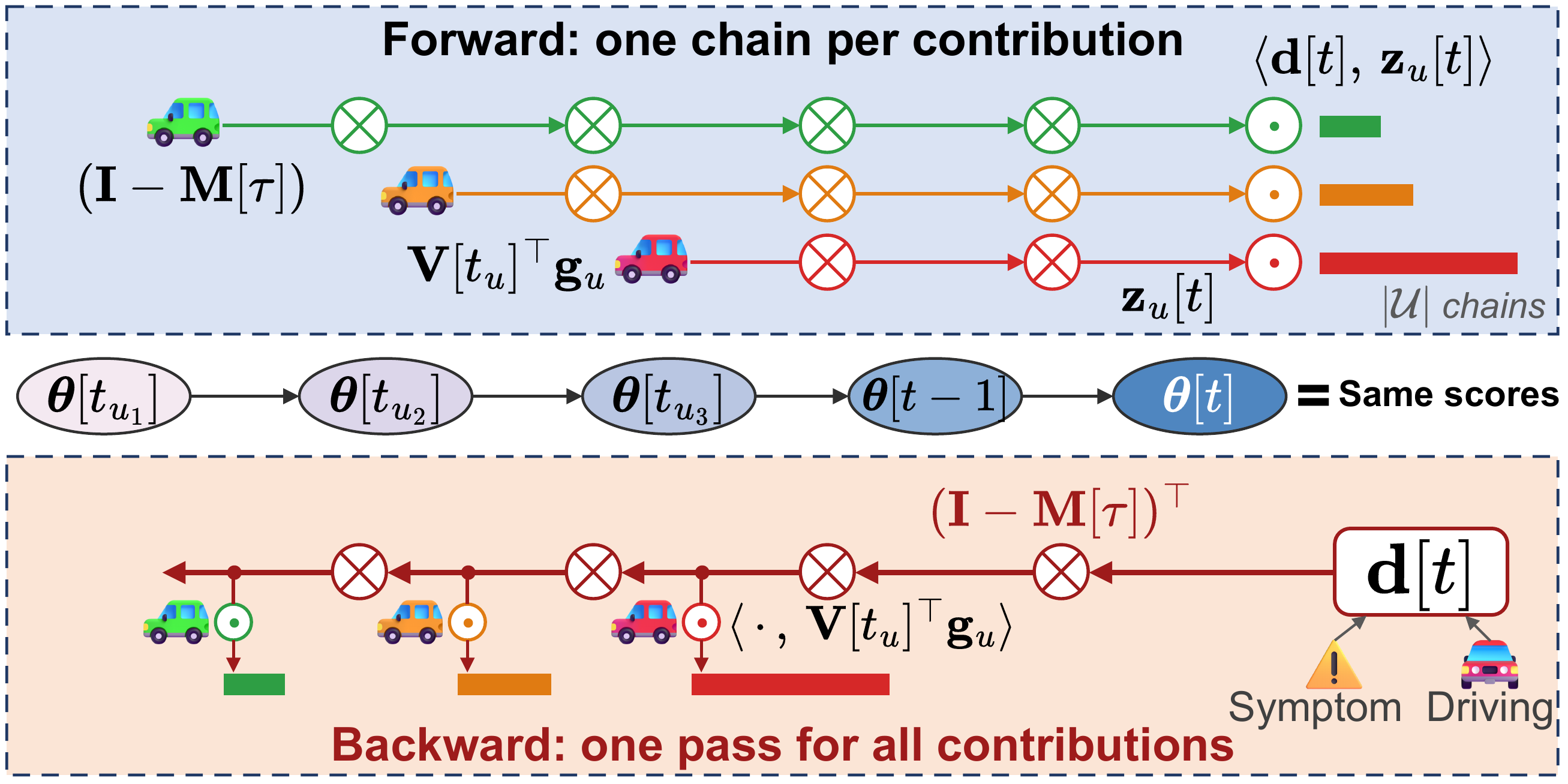}
    \caption{Forward residual propagation and its backward equivalent.}
    \label{fig:backward}
  \end{minipage}
  \hfill
  \begin{minipage}[t]{0.48\textwidth}
    \centering
    \includegraphics[width=\linewidth]{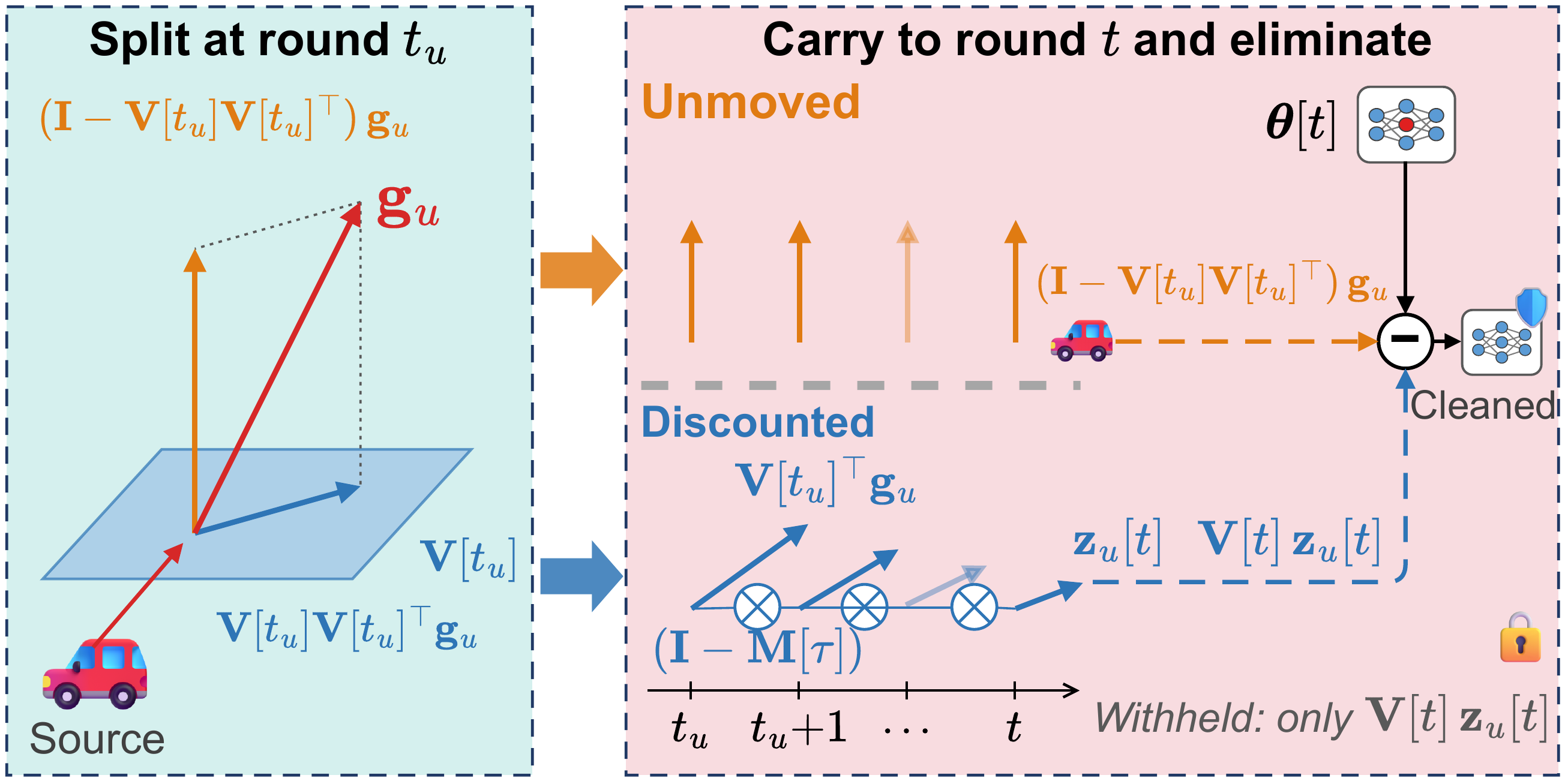}
    \caption{Split of a residual into an unmoved and a discounted part.}
    \label{fig:split}
  \end{minipage}
\end{figure*}

\subsection{System Model}
\label{subsec:system-model}

We consider a fleet of $N$ connected vehicles coordinated by a roadside unit, as illustrated in Fig.~\ref{fig:system}. Each vehicle $k$ maintains a local driving dataset $\mathcal{D}_k$ that remains on board, and $\mathcal{D}=\bigcup_{k=1}^{N}\mathcal{D}_k$ denotes the union of all local datasets. The vehicles collaboratively train a shared driving model $\boldsymbol{\theta}$ through federated learning (FL), with the global objective
\begin{equation}
    \min_{\boldsymbol{\theta}}\;\sum_{k=1}^{N} p_k\,\mathbb{E}_{(\mathbf{x},y)\sim\mathcal{D}_k}\!\left[\ell\!\left(\mathbf{x},y;\boldsymbol{\theta}\right)\right],
    \label{eq:fl-objective}
\end{equation}
where $\ell$ is the per-sample loss and $p_k=|\mathcal{D}_k|/|\mathcal{D}|$ reflects the relative data volume of vehicle $k$. At round $t$, the coordinator distributes $\boldsymbol{\theta}[t]$ to the vehicles currently in range. Each participating vehicle performs local stochastic gradient descent on $\mathcal{D}_k$ to obtain $\boldsymbol{\theta}_k[t]$ and returns the update $\mathbf{g}_k[t]=\boldsymbol{\theta}_k[t]-\boldsymbol{\theta}[t]$. The coordinator then aggregates the received updates as
\begin{equation}
    \boldsymbol{\theta}[t+1]\;=\;\boldsymbol{\theta}[t]+\sum_{k=1}^{N} p_k\,\mathbf{g}_{k}[t].
    \label{eq:fl-aggregation}
\end{equation}
Local epoch counts may vary with on-board computational capacity, while vehicles outside communication range continue operating with older model versions until they reconnect.

The deployed model is also used for driving, allowing operational failures to reveal problems that were not apparent during training. After a near miss, an inconsistent prediction, or a disagreement with another vehicle, a vehicle reports the observed event to the coordinator. We refer to such an observation as a symptom and denote by $\mathcal{S}[1{:}t]$ all symptoms received up to round $t$. A symptom indicates that some learned behavior may be faulty, but does not directly identify the training contribution responsible for it. We use a contribution $u$ to represent the updates $\mathbf{g}_k[t]$ submitted by one vehicle over a window of rounds and within a particular scope, such as an object class or road segment, and let $\mathcal{U}$ denote all contributions incorporated into the model. Through Eq.~\eqref{eq:fl-aggregation}, the effect of a contribution propagates to vehicles that subsequently download the global model. Its holder set $\mathcal{A}_u\subseteq[N]$ can therefore expand across the fleet, making partial removal insufficient to eliminate its effect. We assume that an adversary controls a fraction $f$ of the vehicles. Compromised vehicles may collude, fabricate symptoms, and delay their updates, but cannot impersonate honest vehicles or read their raw data.

\subsection{Problem Formulation}
\label{subsec:problem-formulation}

Conventional federated unlearning (FU) assumes that a removal request specifies what must be forgotten, such as data associated with a departing client or a data owner exercising a deletion right. The unlearning procedure therefore starts with the removal target already known. This assumption does not hold in an autonomous vehicle fleet. A harmful contribution may become visible only after its effect produces an operational symptom, potentially many rounds after Eq.~\eqref{eq:fl-aggregation} has propagated it through the global model. Moreover, the observed symptom contains no explicit correspondence to the vehicle, update window, or data scope responsible for the behavior. The removal target must therefore be inferred from delayed operational evidence rather than supplied externally, which is the central problem considered in this work.

Let $\phi_u\in\{0,1\}$ indicate whether contribution $u$ is faulty, $\rho_u[t]\in[0,1]$ denote the fraction of its effect that remains in $\boldsymbol{\theta}[t]$, and $h(u)$ quantify the harm caused by that effect to the driving task. The residual harm in the fleet at round $t$ is
\begin{equation}
    H[t]\;=\;\sum_{u\in\mathcal{U}}\phi_u\,\rho_u[t]\,h(u).
    \label{eq:residual-harm}
\end{equation}

At round $t$, the coordinator selects a removal set $\mathcal{U}[t]\subseteq\mathcal{U}$ whose effects are to be cleared from their holders. Removing these contributions incurs communication and computation cost $C(\mathcal{U}[t])$ and may reduce task utility by $\Delta(\mathcal{U}[t])$. The objective is to minimize the average residual harm over $T$ rounds while limiting both removal-induced utility degradation and the total removal cost:
\begin{equation}
\begin{split}
    (\mathrm{P}):\;\;\min_{\{\mathcal{U}[t]\}}\;\;
    & \frac{1}{T}\sum_{t=1}^{T} H[t] \\
    \mathrm{s.t.}\;\;
    & \Delta\!\left(\mathcal{U}[t]\right) \le \epsilon, \\
    & \sum_{t=1}^{T} C\!\left(\mathcal{U}[t]\right) \le B.
\end{split}
\label{eq:problem}
\end{equation}
The first constraint prevents indiscriminate removal that sacrifices the driving task, while the second bounds the total communication and computation devoted to removal.

Problem~$(\mathrm{P})$ cannot be solved directly because its key quantities are not observable to the coordinator. In particular, $\phi_u$ is unknown, directly measuring $\rho_u[t]$ would require constructing a counterfactual model trained without $u$, and the coordinator observes only the delayed symptom history $\mathcal{S}[1{:}t]$. These symptoms provide evidence of harmful behavior but no label linking a failure to the contribution that caused it.
\section{\texorpdfstring{\method{}}{ATONE}: Aggregation-Time Operational-Evidence Neutralization of Effects}
\label{sec:method}

To solve $(\mathrm{P})$ in Eq.~\eqref{eq:problem}, we design \underline{\textbf{A}}ggregation-\underline{\textbf{T}}ime \underline{\textbf{O}}perational-evidence \underline{\textbf{N}}eutralization of \underline{\textbf{E}}ffects (\method{}), an in-loop unlearning method whose three modules are described in Sections~\ref{subsec:evidence}, \ref{subsec:propagation}, and~\ref{subsec:correction}. Fig.~\ref{fig:backward} and Fig.~\ref{fig:split} illustrate the two steps that the equations alone leave implicit.

\subsection{Operational Evidence Aggregation}
\label{subsec:evidence}

This module turns the symptoms of Section~\ref{subsec:system-model}, together with the ordinary experience of driving, into one direction the coordinator can act on. A vehicle running local training already differentiates the loss, so stating where the deployed model is wrong costs one further backward pass and no labeled failure case. Reports must nevertheless be small and comparable, which calls for a common low-dimensional subspace.

Let $\mathbf{w}[\tau]=\boldsymbol{\theta}[\tau+1]-\boldsymbol{\theta}[\tau]$ denote the global increment of round $\tau$, which is the aggregated term of Eq.~\eqref{eq:fl-aggregation}. The evidence subspace spans the $q$ most recent increments,
\begin{equation}
    \mathbf{V}[t]=\mathrm{orth}\!\left(\mathbf{W}[t]\right),\;\;
    \mathbf{W}[t]=\left[\mathbf{w}[t-1],\ldots,\mathbf{w}[t-q]\right],
    \label{eq:basis}
\end{equation}
where $q$ is the subspace dimension and $\mathrm{orth}(\cdot)$ returns an orthonormal basis of the column space. Every vehicle rebuilds $\mathbf{V}[t]$ from the consecutive global models it has downloaded, so the subspace is shared without any extra broadcast.

Vehicle $k$ evaluates the loss gradient on its recent driving data and uploads the $q$ projected coefficients alongside its update. The coordinator combines them with the weights of Eq.~\eqref{eq:fl-objective} into the fleet evidence $\mathbf{d}[t]$. A symptom in $\mathcal{S}[1{:}t]$ travels through the same channel and enters $\mathbf{d}[t]$ with a larger weight, so sparse reported failures and dense everyday error accumulate in one quantity of only $q$ coefficients.

\subsection{Residual Propagation and Scoring}
\label{subsec:propagation}

The evidence $\mathbf{d}[t]$ says how the model errs but not which contribution is responsible. The two are connected by the residual $\boldsymbol{\theta}[t]-\boldsymbol{\theta}^{-u}[t]$ of contribution $u$, where $\boldsymbol{\theta}^{-u}[t]$ is the model that would have been obtained had $u$ never been aggregated. Reading it off directly requires that counterfactual model, so \method{} propagates it instead, using one linear operator that all contributions share.

The propagation operator $\mathbf{M}[\tau]\in\mathbb{R}^{q\times q}$ is computed from increments that the coordinator already holds,
\begin{equation}
    \mathbf{M}[\tau] \;=\; -\,\mathbf{V}[\tau]^{\top}\left(\mathbf{W}[\tau+1]-\mathbf{W}[\tau]\right)\mathbf{W}[\tau]^{\dagger}\,\mathbf{V}[\tau],
    \label{eq:curvature}
\end{equation}
where $(\cdot)^{\dagger}$ denotes the pseudo-inverse. Each column of $\mathbf{W}[\tau+1]-\mathbf{W}[\tau]$ is the difference of two consecutive increments. Since every increment approximates a scaled negative gradient, this difference approximates a curvature product, which \textbf{Lemma}~\ref{lem:curvature} makes exact. No Hessian-vector product is computed on board, and the coordinator keeps one operator per window regardless of the fleet size $N$ and of $|\mathcal{U}|$.

Let $\mathbf{g}_u$ collect the weighted updates $p_k\mathbf{g}_k[\tau]$ of contribution $u$, and let $\boldsymbol{\theta}[t_u]$ be the first global model that contains them. Applying Eq.~\eqref{eq:curvature} from $t_u$ to the present gives
\begin{equation}
    \mathbf{z}_u[t] \;=\; \prod_{\tau=t_u}^{t-1}\left(\mathbf{I}-\mathbf{M}[\tau]\right)\mathbf{V}[t_u]^{\top}\mathbf{g}_u,
    \label{eq:propagate}
\end{equation}
where $\mathbf{I}$ is the identity matrix and the product places later rounds on the left. The surviving fraction in Eq.~\eqref{eq:residual-harm} is $\rho_u[t]=\|\mathbf{z}_u[t]\|/\|\mathbf{V}[t_u]^{\top}\mathbf{g}_u\|$, and $u$ is scored by $\langle\mathbf{d}[t],\mathbf{z}_u[t]\rangle$. The score is large when what survives of $u$ points where the fleet errs now, which is residual harm, not its past influence, as \textbf{Proposition}~\ref{prop:harm} in Section~\ref{sec:analysis} shows to first order.

Evaluating Eq.~\eqref{eq:propagate} for every contribution would scale with $|\mathcal{U}|$. The coordinator instead multiplies $\mathbf{d}[t]$ by $(\mathbf{I}-\mathbf{M}[\tau])^{\top}$ from $\tau=t-1$ down to the oldest $t_u$. The score of $u$ is the inner product of the vector reached at $t_u$ with $\mathbf{V}[t_u]^{\top}\mathbf{g}_u$. One backward pass thus serves every contribution, and each score costs one $q$-dimensional inner product. Fig.~\ref{fig:backward} contrasts the two directions, which yield the same score by \textbf{Proposition}~\ref{prop:backward}.

\subsection{Two-Tier In-Loop Correction}
\label{subsec:correction}

This module acts on the model using the scores of Section~\ref{subsec:propagation}. It runs in two tiers, because the two parts of a residual are not recovered at the same cost. The first tier fires in every round and augments Eq.~\eqref{eq:fl-aggregation} with a suppression term,
\begin{equation}
    \boldsymbol{\theta}[t+1]\;=\;\boldsymbol{\theta}[t]+\sum_{k=1}^{N} p_k\,\mathbf{g}_k[t]\;-\;\sum_{u\in\mathcal{U}}\lambda_u[t]\,\mathbf{V}[t]\,\mathbf{z}_u[t],
    \label{eq:suppression}
\end{equation}
where the gain $\lambda_u[t]$ increases with the score of $u$ and is capped so that the per-round utility loss stays within $\epsilon$. This tier needs no trigger and adds nothing to the uplink, but it reaches only the part of the residual lying inside $\mathbf{V}[t]$.

The second tier fires when a score exceeds a threshold $\kappa$. The coordinator broadcasts the candidate counterfactual model for one round and collects from a few vehicles whether the reported error is removed and how far clean accuracy moves. The latter measures $\Delta(\mathcal{U}[t])$ instead of estimating it, and a candidate derived from a fabricated report improves no genuine behavior and is discarded. On confirmation, the coordinator eliminates $u$ by subtracting from $\boldsymbol{\theta}[t]$
\begin{equation}
    \underbrace{\left(\mathbf{I}-\mathbf{V}[t_u]\mathbf{V}[t_u]^{\top}\right)\mathbf{g}_u}_{\text{unmoved}} \;+\; \underbrace{\mathbf{V}[t]\,\mathbf{z}_u[t]}_{\text{discounted}},
    \label{eq:elimination}
\end{equation}
whose two terms are shown in Fig.~\ref{fig:split}. The first is returned once by the source vehicle and subtracted as submitted, because training leaves those directions untouched. The second is discounted so that subtraction does not overshoot. \textbf{Theorem}~\ref{thm:elimination} shows that the two terms together equal the residual of $u$. If the source withholds its update, only the second term is applied, which still clears the harm counted by $H[t]$. The source is then suspended for the affected scope in later rounds.

Beyond the training of Eqs.~\eqref{eq:fl-objective}--\eqref{eq:fl-aggregation}, \method{} adds $q$ coefficients per vehicle per round on the uplink, one model-sized exchange whenever the second tier fires, and one $q\times q$ operator per window on the coordinator. A smaller $\kappa$ eliminates sooner, which shortens the interval over which $H[t]$ accumulates and, by \textbf{Theorem}~\ref{thm:robust}, keeps the elimination error small. The cost is more probes and faster consumption of $B$.

\begin{figure*}[t]
  \centering
  \begin{minipage}[t]{0.48\textwidth}
    \centering
    \includegraphics[width=\linewidth]{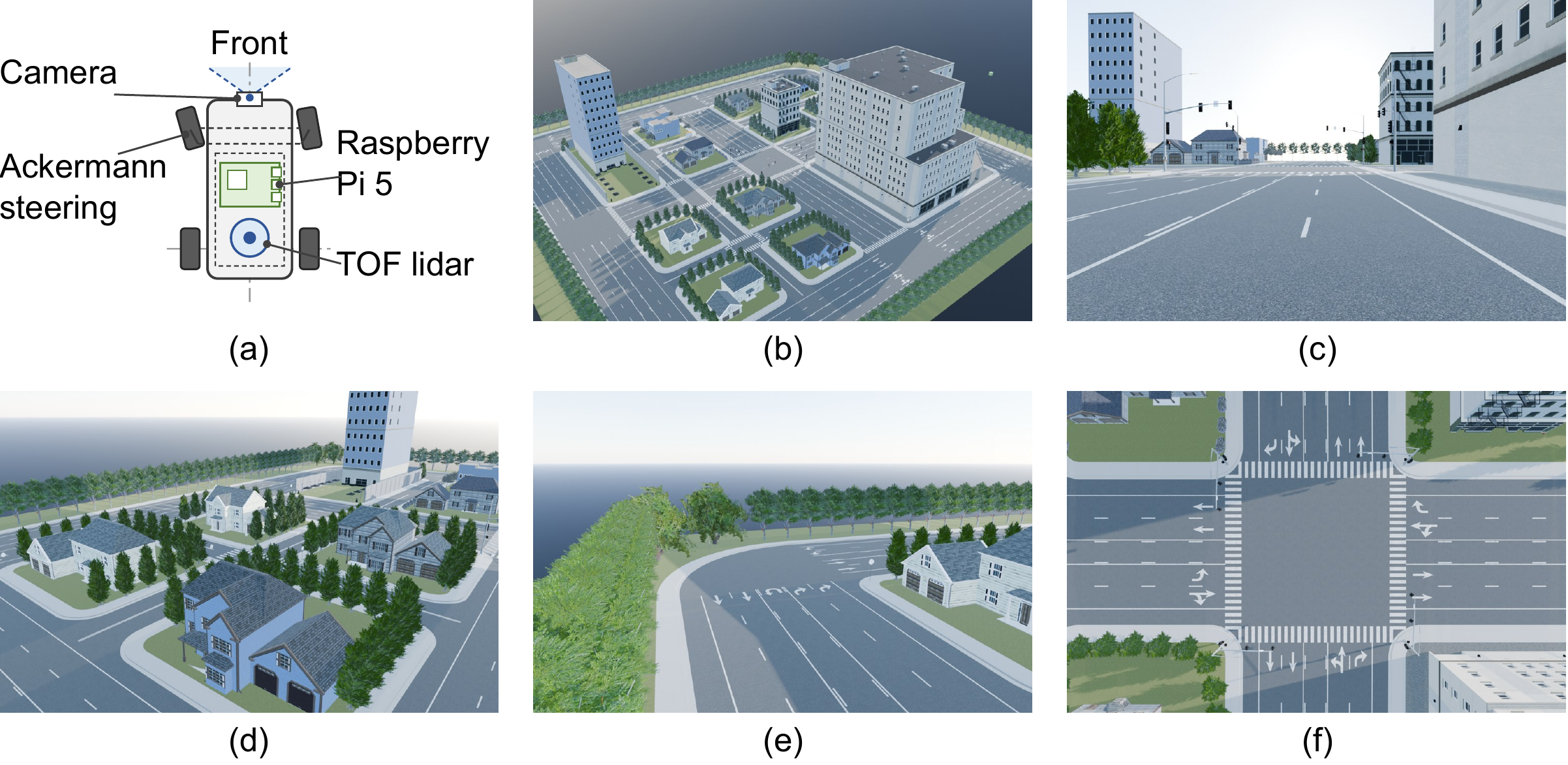}
    \caption{(a) Top view of the MentorPi robot car. (b)--(f) Simulated city: overview, vehicle view, residential block, curve, and intersection.}
    \label{fig:platform}
  \end{minipage}
  \hfill
  \begin{minipage}[t]{0.48\textwidth}
    \centering
    \includegraphics[width=\linewidth]{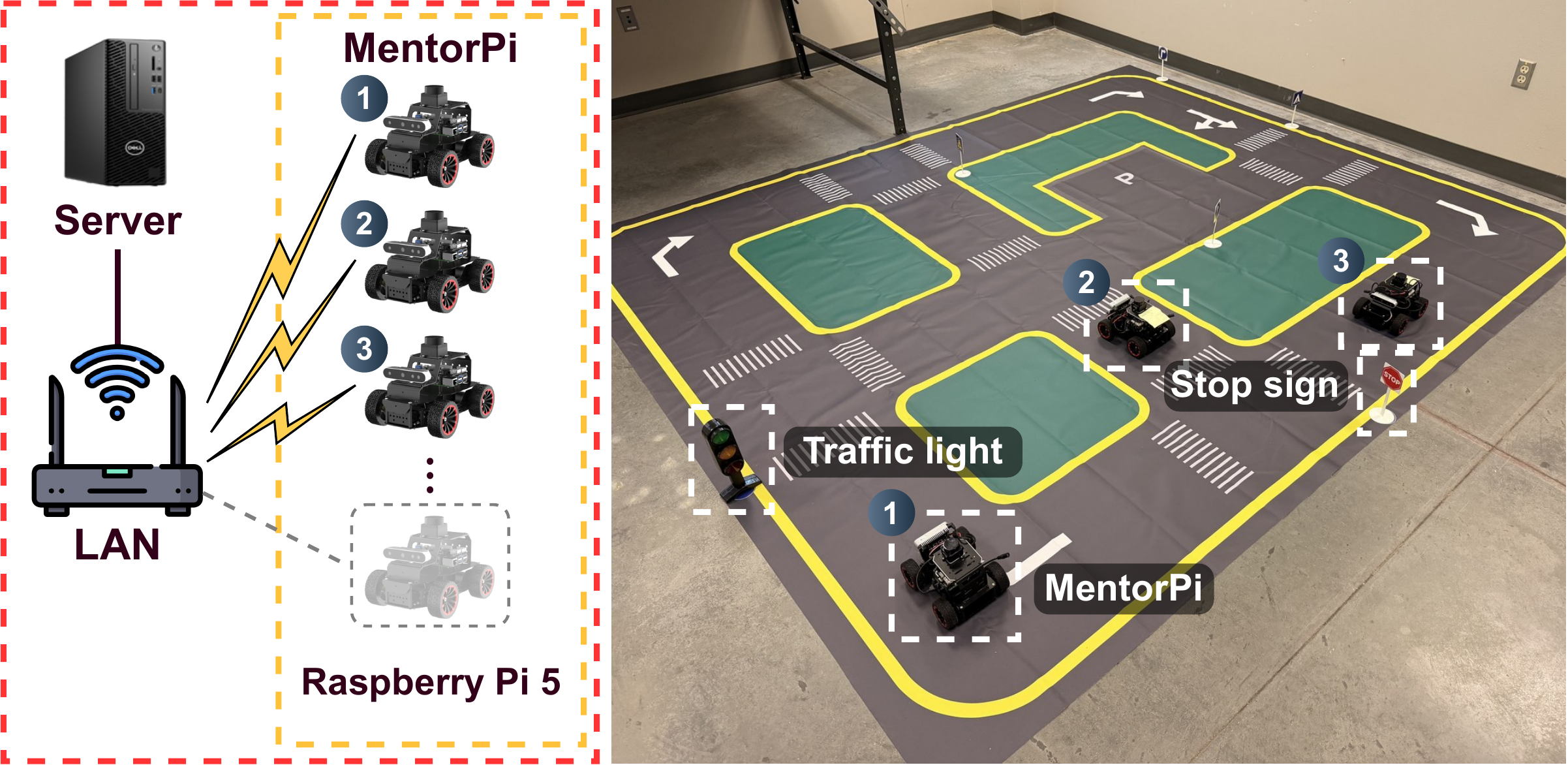}
    \caption{Hardware testbed. MentorPi robot cars exchange models and updates with a server over a LAN while driving on a physical track.}
    \label{fig:testbed}
  \end{minipage}
\end{figure*}

\begin{table*}[!t]
\centering
\setlength{\tabcolsep}{3pt}
\setlength{\aboverulesep}{0pt}
\setlength{\belowrulesep}{0pt}
\setlength{\extrarowheight}{0pt}
\renewcommand{\arraystretch}{0.88}
\caption{Unlearning and residual harm per method and dataset, and removal cost on the testbed. $\mathrm{FmIoU}$, $\mathrm{RmIoU}$, $\mathrm{TmIoU}$, and $\mathrm{MIA}$ are in \%, $\bar{H}$ in percentage points, and the delay in rounds. Each value is the mean over five random seeds with the standard deviation in parentheses. The testbed columns come from the Raspberry Pi 5 cars trained on Foggy Cityscapes, with the time in seconds and the uplink in MB that the removal adds. $\mathrm{FmIoU}$ and $\mathrm{MIA}$ are best when closest to Retrain of the same unlearning type. A \swatch{bestcell} cell marks the best value and a \swatch{secondcell} cell the second best within each unlearning type, excluding Retrain.}
\label{tab:main}
\resizebox{\linewidth}{!}{\begin{tabular}{l|cccccc|cccccc|cccccc|cccc}
\Xhline{1.2pt}
 & \multicolumn{6}{c|}{\hd{hdrA}{Foggy Cityscapes}} & \multicolumn{6}{c|}{\hd{hdrB}{ACDC}} & \multicolumn{6}{c|}{\hd{hdrC}{BDD100K}} & \multicolumn{4}{c}{\hd{hdrD}{Testbed}}\\
\textbf{Method} & FmIoU & RmIoU$\uparrow$ & TmIoU$\uparrow$ & MIA & $\bar{H}{\downarrow}$ & Delay$\downarrow$ & FmIoU & RmIoU$\uparrow$ & TmIoU$\uparrow$ & MIA & $\bar{H}{\downarrow}$ & Delay$\downarrow$ & FmIoU & RmIoU$\uparrow$ & TmIoU$\uparrow$ & MIA & $\bar{H}{\downarrow}$ & Delay$\downarrow$ & FmIoU & RmIoU$\uparrow$ & Time$\downarrow$ & Uplink$\downarrow$\\
\Xhline{1.2pt}
\multicolumn{23}{l}{\band{(a) Client unlearning}}\\
\Xhline{0.6pt}
FedRecover \venue{S\&P'23}~\cite{Cao2023SP} & 0.18\std{0.03} & 64.80\std{0.64} & 64.03\std{0.57} & 56.2\std{2.5} & 0.46\std{0.07} & 7.2\std{1.1} & 0.32\std{0.04} & 63.98\std{0.36} & 63.93\std{0.58} & 56.5\std{12.2} & 0.58\std{0.08} & 8.5\std{1.5} & 0.30\std{0.04} & 64.38\std{0.57} & 63.74\std{0.59} & 57.7\std{6.7} & 0.38\std{0.06} & 6.8\std{1.1} & 0.22\std{0.04} & 64.36\std{0.68} & 192.4\std{16.7} & 307.964\std{26.788}\\
FedOSD \venue{AAAI'25}~\cite{Pan2025AAAI} & 0.29\std{0.05} & 63.47\std{0.91} & 63.07\std{0.77} & 61.3\std{3.4} & 0.71\std{0.12} & 11.4\std{1.8} & 1.12\std{0.18} & 63.56\std{0.37} & 62.66\std{0.71} & 61.2\std{14.7} & 1.26\std{0.20} & 17.6\std{2.4} & 0.41\std{0.06} & 63.72\std{0.93} & 63.39\std{0.51} & 60.0\std{6.6} & 0.63\std{0.11} & 10.2\std{1.8} & 0.34\std{0.07} & 62.87\std{0.81} & 68.7\std{6.4} & \second{83.768\std{9.111}}\\
FedEditor \venue{TIFS'25}~\cite{Yuan2025TIFS} & \second{0.13\std{0.03}} & \second{65.21\std{0.73}} & \second{64.82\std{0.83}} & \second{53.8\std{2.5}} & \second{0.33\std{0.05}} & \second{5.6\std{1.0}} & \second{0.27\std{0.06}} & \second{64.76\std{0.89}} & \second{64.75\std{0.89}} & \second{51.8\std{12.0}} & \second{0.40\std{0.05}} & \second{6.2\std{0.6}} & \second{0.26\std{0.04}} & \second{65.66\std{0.61}} & \second{64.85\std{0.57}} & \second{52.8\std{7.2}} & \second{0.29\std{0.05}} & \second{4.9\std{0.9}} & \second{0.16\std{0.03}} & \second{65.12\std{0.74}} & \best{35.6\std{2.8}} & \best{27.074\std{9.507}}\\
FLForensics \venue{NeurIPS'25}~\cite{Jia2025NeurIPS} & 0.26\std{0.05} & 64.31\std{0.42} & 63.84\std{0.35} & 59.8\std{2.8} & 1.05\std{0.15} & 16.3\std{2.3} & 0.37\std{0.04} & 64.28\std{0.28} & 63.55\std{0.55} & 61.8\std{11.8} & 1.44\std{0.23} & 19.7\std{3.5} & 0.36\std{0.08} & 64.60\std{0.47} & 63.96\std{0.73} & 60.4\std{6.1} & 0.93\std{0.15} & 14.6\std{1.8} & 0.32\std{0.06} & 64.07\std{0.63} & 261.8\std{24.6} & 353.061\std{36.545}\\
FLTracer \venue{TIFS'24}~\cite{Zhang2024TIFS} & 0.23\std{0.03} & 64.59\std{0.87} & 63.87\std{0.42} & 57.3\std{3.2} & 0.83\std{0.15} & 13.8\std{2.4} & 0.40\std{0.05} & 64.05\std{0.80} & 63.64\std{0.78} & 58.2\std{10.2} & 1.09\std{0.19} & 15.4\std{1.8} & 0.35\std{0.07} & 64.57\std{0.83} & 64.04\std{0.70} & 58.6\std{7.4} & 0.74\std{0.11} & 12.3\std{1.4} & 0.29\std{0.05} & 64.48\std{0.59} & 398.5\std{37.2} & 409.755\std{31.444}\\
\Xhline{0.8pt}
\cellcolor{oursname}\textit{\textbf{\methodname{}}} & \best{0.11\std{0.02}} & \best{65.85\std{0.78}} & \best{65.25\std{0.40}} & \best{51.5\std{2.6}} & \best{0.14\std{0.02}} & \best{3.1\std{0.6}} & \best{0.25\std{0.05}} & \best{65.62\std{0.33}} & \best{64.94\std{0.36}} & \best{50.0\std{10.5}} & \best{0.23\std{0.03}} & \best{4.3\std{0.8}} & \best{0.25\std{0.06}} & \best{66.10\std{0.83}} & \best{65.57\std{0.64}} & \best{50.5\std{6.5}} & \best{0.12\std{0.01}} & \best{2.7\std{0.5}} & \best{0.12\std{0.02}} & \best{65.76\std{0.67}} & \second{54.8\std{4.2}} & 99.229\std{13.650}\\
\Xhline{0.8pt}
\textit{Retrain} & 0.09\std{0.02} & 66.21\std{0.44} & 65.47\std{0.74} & 50.4\std{3.1} & 0.00 & - & 0.20\std{0.03} & 65.82\std{0.28} & 65.24\std{0.94} & 49.4\std{9.7} & 0.00 & - & 0.20\std{0.03} & 66.32\std{0.93} & 65.73\std{0.74} & 49.8\std{5.2} & 0.00 & - & 0.11\std{0.03} & 66.09\std{0.71} & - & -\\
\Xhline{0.8pt}
\multicolumn{23}{l}{\band{(b) Class unlearning}}\\
\Xhline{0.6pt}
FedRecover \venue{S\&P'23}~\cite{Cao2023SP} & 45.81\std{0.89} & 64.01\std{0.36} & 64.23\std{0.91} & 56.5\std{2.9} & 1.13\std{0.14} & 14.6\std{1.8} & 22.62\std{0.79} & 64.20\std{0.85} & 63.15\std{0.67} & 58.8\std{8.5} & 1.65\std{0.27} & 18.4\std{2.7} & 38.56\std{0.53} & 64.68\std{0.35} & 64.27\std{0.39} & 56.5\std{5.9} & 0.97\std{0.15} & 12.4\std{2.4} & 45.72\std{0.64} & 63.96\std{0.73} & 138.2\std{8.6} & 222.924\std{20.191}\\
FedOSD \venue{AAAI'25}~\cite{Pan2025AAAI} & 46.33\std{0.70} & 63.19\std{0.87} & 63.60\std{0.57} & 60.7\std{3.5} & 1.57\std{0.28} & 19.3\std{2.9} & 23.20\std{0.91} & 63.06\std{0.84} & 62.83\std{0.44} & 65.4\std{8.7} & 2.17\std{0.53} & 25.1\std{4.1} & 39.07\std{0.47} & 63.90\std{0.67} & 63.18\std{0.79} & 65.7\std{6.8} & 1.36\std{0.16} & 17.8\std{2.5} & 46.58\std{0.48} & 63.28\std{0.57} & 76.3\std{6.2} & \second{67.017\std{13.307}}\\
FedEditor \venue{TIFS'25}~\cite{Yuan2025TIFS} & \second{45.46\std{0.61}} & \second{65.55\std{0.49}} & \best{65.71\std{0.94}} & \best{50.0\std{3.0}} & \best{0.22\std{0.02}} & \best{3.7\std{0.5}} & \second{22.24\std{0.69}} & \second{64.70\std{0.70}} & \second{64.52\std{0.86}} & \second{55.0\std{9.6}} & \second{0.82\std{0.16}} & \second{10.3\std{1.6}} & \second{38.27\std{0.41}} & \second{65.40\std{0.92}} & \second{64.83\std{0.32}} & \second{54.3\std{5.5}} & \second{0.49\std{0.08}} & \second{7.6\std{1.0}} & \best{45.36\std{0.46}} & \second{65.17\std{0.68}} & \best{42.7\std{3.9}} & \best{37.382\std{9.507}}\\
FLForensics \venue{NeurIPS'25}~\cite{Jia2025NeurIPS} & 46.72\std{0.46} & 64.50\std{0.58} & 63.91\std{0.80} & 58.7\std{3.9} & 1.82\std{0.24} & 29.4\std{3.2} & 23.70\std{0.65} & 64.10\std{0.53} & 63.61\std{0.42} & 62.7\std{9.4} & 2.66\std{0.41} & 36.7\std{5.5} & 39.46\std{0.27} & 64.71\std{0.84} & 63.76\std{0.32} & 59.3\std{5.3} & 1.64\std{0.29} & 26.1\std{4.8} & 46.91\std{0.73} & 64.08\std{0.51} & 297.6\std{25.8} & 247.405\std{20.865}\\
FLTracer \venue{TIFS'24}~\cite{Zhang2024TIFS} & 46.17\std{0.69} & 64.35\std{0.61} & 63.58\std{0.56} & 58.6\std{4.0} & 1.29\std{0.27} & 22.6\std{3.2} & 23.08\std{0.41} & 64.07\std{0.93} & 63.29\std{0.37} & 59.2\std{9.1} & 1.95\std{0.23} & 28.9\std{4.0} & 38.96\std{0.58} & 64.19\std{0.35} & 63.59\std{0.44} & 58.5\std{5.3} & 1.16\std{0.20} & 20.4\std{3.0} & 46.24\std{0.57} & 64.51\std{0.76} & 436.8\std{41.3} & 289.925\std{39.011}\\
\Xhline{0.8pt}
\cellcolor{oursname}\textit{\textbf{\methodname{}}} & \best{45.39\std{0.63}} & \best{65.86\std{0.35}} & \second{65.52\std{0.26}} & \second{50.6\std{3.5}} & \second{0.28\std{0.05}} & \second{4.8\std{0.9}} & \best{22.13\std{0.67}} & \best{65.45\std{0.62}} & \best{64.86\std{0.69}} & \best{51.5\std{12.2}} & \best{0.35\std{0.06}} & \best{5.7\std{1.1}} & \best{38.19\std{0.52}} & \best{65.72\std{0.73}} & \best{65.33\std{0.68}} & \best{51.7\std{8.3}} & \best{0.19\std{0.03}} & \best{4.1\std{0.6}} & \second{45.42\std{0.55}} & \best{66.02\std{0.81}} & \second{69.3\std{5.7}} & 166.231\std{11.282}\\
\Xhline{0.8pt}
\textit{Retrain} & 45.27\std{0.45} & 66.17\std{0.58} & 65.78\std{0.74} & 49.7\std{3.5} & 0.00 & - & 22.00\std{0.35} & 65.68\std{0.49} & 65.09\std{0.37} & 50.8\std{7.9} & 0.00 & - & 38.10\std{0.49} & 66.08\std{0.78} & 65.56\std{0.81} & 51.1\std{9.3} & 0.00 & - & 45.31\std{0.39} & 66.23\std{0.72} & - & -\\
\Xhline{0.8pt}
\multicolumn{23}{l}{\band{(c) Sample unlearning}}\\
\Xhline{0.6pt}
FedRecover \venue{S\&P'23}~\cite{Cao2023SP} & 42.28\std{0.55} & 63.88\std{0.68} & 63.50\std{0.63} & 57.7\std{5.7} & 0.63\std{0.07} & 9.2\std{1.0} & 29.99\std{0.88} & 64.55\std{0.80} & 63.94\std{0.84} & 60.0\std{21.1} & 0.76\std{0.09} & 10.8\std{1.4} & 37.39\std{0.54} & 64.97\std{0.54} & 64.14\std{0.66} & 56.0\std{22.7} & 0.54\std{0.06} & 8.6\std{1.0} & 42.41\std{0.62} & 63.92\std{0.79} & 104.3\std{8.1} & 146.903\std{16.298}\\
FedOSD \venue{AAAI'25}~\cite{Pan2025AAAI} & 42.64\std{0.73} & 63.53\std{0.84} & 62.51\std{0.87} & 65.2\std{4.6} & 0.89\std{0.16} & 13.6\std{2.4} & 30.45\std{0.27} & 63.23\std{0.45} & 62.42\std{0.74} & 70.0\std{29.2} & 1.14\std{0.17} & 15.7\std{1.9} & 37.64\std{0.38} & 63.79\std{0.83} & 63.03\std{0.27} & 60.0\std{13.3} & 0.78\std{0.09} & 12.7\std{1.6} & 42.89\std{0.47} & 63.18\std{0.66} & 59.2\std{5.4} & \second{47.690\std{14.940}}\\
FedEditor \venue{TIFS'25}~\cite{Yuan2025TIFS} & \second{41.83\std{0.51}} & \second{65.11\std{0.86}} & \second{64.93\std{0.66}} & \second{54.2\std{7.4}} & \second{0.44\std{0.09}} & \second{7.3\std{1.1}} & \second{29.47\std{0.41}} & \best{65.83\std{0.82}} & \second{64.53\std{0.65}} & \second{56.7\std{27.4}} & \second{0.55\std{0.08}} & \second{8.1\std{1.4}} & \second{36.91\std{0.30}} & \second{65.81\std{0.61}} & \second{65.05\std{0.63}} & \second{52.0\std{19.3}} & \second{0.37\std{0.05}} & \second{6.1\std{1.0}} & \second{41.96\std{0.58}} & \second{65.09\std{0.53}} & \best{31.4\std{2.3}} & \best{29.651\std{8.697}}\\
FLForensics \venue{NeurIPS'25}~\cite{Jia2025NeurIPS} & 44.93\std{0.85} & 64.44\std{0.84} & 63.32\std{0.69} & 62.3\std{5.3} & 1.38\std{0.25} & 24.8\std{3.8} & 33.08\std{0.27} & 63.99\std{0.67} & 63.70\std{0.42} & 66.7\std{15.7} & 1.60\std{0.21} & 27.2\std{4.6} & 39.75\std{0.85} & 64.64\std{0.40} & 64.62\std{0.82} & 58.0\std{17.5} & 1.22\std{0.19} & 21.9\std{3.0} & 44.52\std{0.69} & 64.47\std{0.72} & 128.6\std{10.7} & 162.365\std{13.851}\\
FLTracer \venue{TIFS'24}~\cite{Zhang2024TIFS} & 44.21\std{0.63} & 64.13\std{0.44} & 63.69\std{0.41} & 59.8\std{6.2} & 1.17\std{0.24} & 19.1\std{3.6} & 32.24\std{0.74} & 63.86\std{0.46} & 63.13\std{0.32} & 63.3\std{24.6} & 1.43\std{0.20} & 22.5\std{2.8} & 39.11\std{0.77} & 65.12\std{0.47} & 63.92\std{0.64} & 54.0\std{23.2} & 0.99\std{0.20} & 17.3\std{2.4} & 44.13\std{0.61} & 63.94\std{0.58} & 229.4\std{18.6} & 172.673\std{26.615}\\
\Xhline{0.8pt}
\cellcolor{oursname}\textit{\textbf{\methodname{}}} & \best{41.69\std{0.57}} & \best{66.04\std{0.62}} & \best{65.58\std{0.50}} & \best{52.1\std{5.2}} & \best{0.17\std{0.02}} & \best{3.6\std{0.7}} & \best{29.36\std{0.49}} & \second{65.65\std{0.76}} & \best{65.08\std{0.40}} & \best{53.3\std{32.2}} & \best{0.25\std{0.03}} & \best{4.6\std{0.7}} & \best{36.81\std{0.61}} & \best{66.56\std{0.43}} & \best{65.98\std{0.79}} & \best{50.0\std{21.6}} & \best{0.16\std{0.02}} & \best{3.4\std{0.7}} & \best{41.79\std{0.54}} & \best{65.88\std{0.71}} & \second{47.1\std{3.6}} & 67.017\std{8.149}\\
\Xhline{0.8pt}
\textit{Retrain} & 41.57\std{0.60} & 65.87\std{0.77} & 65.41\std{0.35} & 51.1\std{7.5} & 0.00 & - & 29.20\std{0.45} & 65.94\std{0.34} & 65.36\std{0.68} & 50.0\std{23.6} & 0.00 & - & 36.70\std{0.68} & 66.47\std{0.85} & 65.91\std{0.43} & 48.0\std{16.9} & 0.00 & - & 41.62\std{0.59} & 65.97\std{0.83} & - & -\\
\Xhline{1.2pt}
\end{tabular}}
\end{table*}

\section{Theoretical Analysis}
\label{sec:analysis}

This section states when the steps of Section~\ref{sec:method} are exact.

\begin{assumption}\label{ass:gradient}
Let $L(\boldsymbol{\theta})=\sum_{k=1}^{N}p_k\,\mathbb{E}_{(\mathbf{x},y)\sim\mathcal{D}_k}[\ell(\mathbf{x},y;\boldsymbol{\theta})]$ denote the objective of Eq.~\eqref{eq:fl-objective}. The aggregation of Eq.~\eqref{eq:fl-aggregation} and the counterfactual training that yields $\boldsymbol{\theta}^{-u}[\tau]$ both act as gradient steps $\mathbf{w}[\tau]=-\eta\nabla L(\boldsymbol{\theta}[\tau])$, and $L$ is quadratic with a fixed Hessian $\nabla^{2}L(\boldsymbol{\theta}[\tau])$ over any $q+1$ consecutive rounds.
\end{assumption}

\begin{lemma}\label{lem:curvature}
Under \textbf{Assumption}~\ref{ass:gradient}, Eq.~\eqref{eq:curvature} gives $\mathbf{M}[\tau]=\eta\,\mathbf{V}[\tau]^{\top}\nabla^{2}L(\boldsymbol{\theta}[\tau])\,\mathbf{V}[\tau]$, the curvature of $L$ inside the evidence subspace.
\end{lemma}

\begin{proof}
By \textbf{Assumption}~\ref{ass:gradient}, the increments satisfy $\mathbf{W}[\tau+1]-\mathbf{W}[\tau]=-\eta\nabla^{2}L(\boldsymbol{\theta}[\tau])\,\mathbf{W}[\tau]$. Substituting this into Eq.~\eqref{eq:curvature} and using $\mathbf{W}[\tau]\mathbf{W}[\tau]^{\dagger}\mathbf{V}[\tau]=\mathbf{V}[\tau]$ proves the claim.\end{proof}

\begin{assumption}\label{ass:confined}
For $t_u\le\tau\le t$, the evidence subspace is fixed, $\mathbf{V}[\tau]=\mathbf{V}[t_u]$, and $\nabla^{2}L(\boldsymbol{\theta}[\tau])$ maps every direction orthogonal to $\mathbf{V}[t_u]$ to zero, so training is flat there.
\end{assumption}

\begin{lemma}\label{lem:split}
Under \textbf{Assumptions}~\ref{ass:gradient} and~\ref{ass:confined}, for $t_u\le\tau<t$,
\begin{equation}
\begin{split}
    &\boldsymbol{\theta}[\tau+1]-\boldsymbol{\theta}^{-u}[\tau+1] \\
    &\;=\;\left(\mathbf{I}-\mathbf{V}[t_u]\,\mathbf{M}[\tau]\,\mathbf{V}[t_u]^{\top}\right)\left(\boldsymbol{\theta}[\tau]-\boldsymbol{\theta}^{-u}[\tau]\right).
\end{split}
\label{eq:residual-step}
\end{equation}
Hence the residual part orthogonal to $\mathbf{V}[t_u]$ never changes, and its coefficients inside $\mathbf{V}[t_u]$ are scaled by $\mathbf{I}-\mathbf{M}[\tau]$.
\end{lemma}

\begin{proof}
By \textbf{Assumption}~\ref{ass:gradient}, the residual is multiplied by $\mathbf{I}-\eta\nabla^{2}L(\boldsymbol{\theta}[\tau])$ in each round, and \textbf{Assumption}~\ref{ass:confined} with \textbf{Lemma}~\ref{lem:curvature} gives $\eta\nabla^{2}L(\boldsymbol{\theta}[\tau])=\mathbf{V}[t_u]\mathbf{M}[\tau]\mathbf{V}[t_u]^{\top}$.\end{proof}

According to \textbf{Lemma}~\ref{lem:split}, an effect that persists long after $t_u$ lies mostly outside the subspace, where \method{} subtracts it as submitted.

\begin{proposition}\label{prop:backward}
For every $u\in\mathcal{U}$,
\begin{equation}
    \left\langle\mathbf{d}[t],\mathbf{z}_u[t]\right\rangle
    =\Big\langle\prod_{\tau=t_u}^{t-1}\left(\mathbf{I}-\mathbf{M}[\tau]\right)^{\top}\mathbf{d}[t],\ \mathbf{V}[t_u]^{\top}\mathbf{g}_u\Big\rangle,
    \label{eq:backward}
\end{equation}
where the product now places earlier rounds on the left.
\end{proposition}

\begin{proof}
Substitute Eq.~\eqref{eq:propagate} and move each factor to the other slot of the inner product as its transpose.\end{proof}

According to \textbf{Proposition}~\ref{prop:backward}, the first slot depends on $u$ only through $t_u$, so one backward pass returns the score of every contribution without any assumption.

\begin{theorem}\label{thm:elimination}
Under \textbf{Assumptions}~\ref{ass:gradient} and~\ref{ass:confined}, the quantity in Eq.~\eqref{eq:elimination} equals $\boldsymbol{\theta}[t]-\boldsymbol{\theta}^{-u}[t]$, so subtracting it from $\boldsymbol{\theta}[t]$ returns exactly $\boldsymbol{\theta}^{-u}[t]$.
\end{theorem}

\textbf{Theorem}~\ref{thm:elimination} follows from \textbf{Lemma}~\ref{lem:split}. The residual starts from $\mathbf{g}_u$ at $t_u$ by Eq.~\eqref{eq:fl-aggregation}, and \textbf{Lemma}~\ref{lem:split} applied from $\tau=t_u$ to $t-1$ keeps its orthogonal part and turns $\mathbf{V}[t_u]^{\top}\mathbf{g}_u$ into $\mathbf{z}_u[t]$.

To relate the score to harm, let $\psi$ be the evidence loss behind $\mathbf{d}[t]$, so that $\mathbf{d}[t]=\mathbf{V}[t]^{\top}\nabla\psi(\boldsymbol{\theta}[t])$.

\begin{proposition}\label{prop:harm}
If $\psi$ is twice differentiable, then under \textbf{Assumptions}~\ref{ass:gradient} and~\ref{ass:confined},
\begin{equation}
\begin{split}
    \psi(\boldsymbol{\theta}[t])-\psi(\boldsymbol{\theta}^{-u}[t])
    &=\left\langle\mathbf{d}[t],\mathbf{z}_u[t]\right\rangle+O\big(\|\mathbf{g}_u\|^{2}\big) \\
    &\;\;+\big\langle(\mathbf{I}-\mathbf{V}[t_u]\mathbf{V}[t_u]^{\top})\nabla\psi(\boldsymbol{\theta}[t]),\,\mathbf{g}_u\big\rangle.
\end{split}
\label{eq:harm}
\end{equation}
\end{proposition}

\begin{proof}
Expand $\psi$ to first order around $\boldsymbol{\theta}[t]$ and insert the residual of \textbf{Theorem}~\ref{thm:elimination}. Its discounted term gives $\langle\mathbf{d}[t],\mathbf{z}_u[t]\rangle$, and its unmoved term gives the last term.\end{proof}

According to \textbf{Proposition}~\ref{prop:harm}, the score is the first-order harm of $u$ inside the evidence subspace. The last term needs $\mathbf{g}_u$ itself, which is why the second tier probes before it eliminates.

\begin{theorem}\label{thm:robust}
Let \textbf{Assumption}~\ref{ass:gradient} hold and $\mathbf{V}[\tau]=\mathbf{V}[t_u]$ for $t_u\le\tau\le t$. In place of the flatness in \textbf{Assumption}~\ref{ass:confined}, suppose that $\mathbf{0}\preceq\eta\nabla^{2}L(\boldsymbol{\theta}[\tau])\preceq 2\mathbf{I}$ and
\begin{equation}
    \big\|\eta\nabla^{2}L(\boldsymbol{\theta}[\tau])-\mathbf{V}[t_u]\,\mathbf{M}[\tau]\,\mathbf{V}[t_u]^{\top}\big\|\;\le\;\delta
    \label{eq:leak}
\end{equation}
for $t_u\le\tau<t$. Then the quantity in Eq.~\eqref{eq:elimination} differs from $\boldsymbol{\theta}[t]-\boldsymbol{\theta}^{-u}[t]$ by at most $(t-t_u)\,\delta\,\|\mathbf{g}_u\|$.
\end{theorem}

\begin{proof}
Let $\boldsymbol{\xi}[\tau]$ be the gap between $\boldsymbol{\theta}[\tau]-\boldsymbol{\theta}^{-u}[\tau]$ and Eq.~\eqref{eq:elimination} at round $\tau$, so $\boldsymbol{\xi}[t_u]=\mathbf{0}$. Per round, the two are multiplied by the non-expansive $\mathbf{I}-\eta\nabla^{2}L(\boldsymbol{\theta}[\tau])$ and by $\mathbf{I}-\mathbf{V}[t_u]\mathbf{M}[\tau]\mathbf{V}[t_u]^{\top}$. Their difference, bounded by Eq.~\eqref{eq:leak}, acts on Eq.~\eqref{eq:elimination}, whose norm stays at most $\|\mathbf{g}_u\|$. Hence $\|\boldsymbol{\xi}[\tau+1]\|\le\|\boldsymbol{\xi}[\tau]\|+\delta\|\mathbf{g}_u\|$.\end{proof}

According to \textbf{Theorem}~\ref{thm:robust}, the elimination error grows at most linearly with the rounds elapsed since $t_u$, so a smaller $\kappa$ that eliminates sooner also eliminates more accurately.

\section{Performance Evaluation}
\label{sec:exp}

\subsection{Experimental Setup}
\label{subsec:setup}

We simulate the fleet in CARLA on a city map built with RoadRunner, shown in Fig.~\ref{fig:platform}(b)--(f). $N=50$ vehicles drive in this city at the same time, and a vehicle joins round $t$ only within the range of the RSU, which it reaches over IEEE 802.11. The driving model $\boldsymbol{\theta}$ is an LR-ASPP segmentation network with a MobileNetV3-Large backbone. It is trained separately on Foggy Cityscapes~\cite{Sakaridis2018IJCV}, ACDC~\cite{Sakaridis2021ICCV}, and BDD100K~\cite{Yu2020CVPR}, three driving datasets with the same 19 classes that span fog, other adverse weather, and many cities. Each vehicle receives a different share of the training split. The faulty contribution relabels pedestrians as road in the images of one vehicle during one window. The validation split is halved. One half serves as recent driving data, and a symptom is reported when the model labels an annotated pedestrian as road. The other half is held out for evaluation. On the physical track of Fig.~\ref{fig:testbed}, MentorPi robot cars train on Foggy Cityscapes with a Raspberry Pi 5 and exchange models and updates with a server over a local area network (LAN).

\begin{figure*}[!t]
\centering
\subfloat[Residual harm over rounds.]{\includegraphics[width=0.24\textwidth]{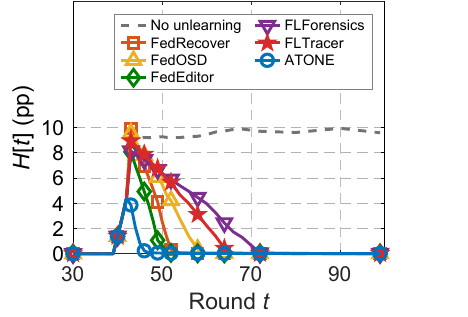}\label{fig:harm}}
\hfill
\subfloat[Top-1 attribution accuracy.]{\includegraphics[width=0.24\textwidth]{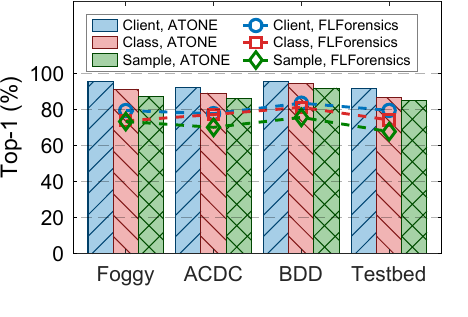}\label{fig:attribution}}
\hfill
\subfloat[Cost and delay versus $\kappa$.]{\includegraphics[width=0.24\textwidth]{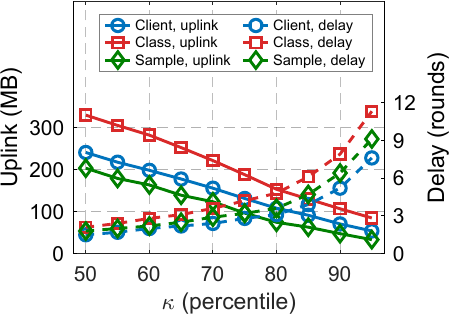}\label{fig:kappa}}
\hfill
\subfloat[Two-tier ablation.]{\includegraphics[width=0.24\textwidth]{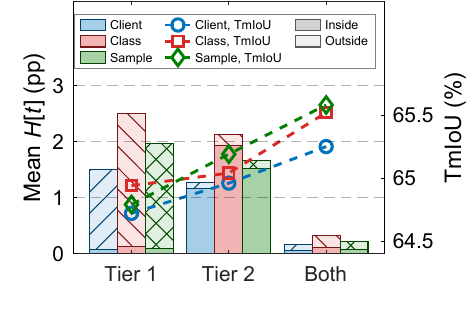}\label{fig:ablation}}
\caption{Results of \method{}. (a) Residual harm $H[t]$ under sample unlearning on Foggy Cityscapes. (b) Share of symptoms whose faulty contribution is ranked first under each unlearning type. (c) Uplink cost and reaction delay as $\kappa$ varies. (d) $\bar{H}$ inside and outside the subspace and $\mathrm{TmIoU}$ with one tier or both.}
\label{fig:results}
\end{figure*}

We adapt the FU metrics to segmentation. Unlearning efficacy is measured by the forget mIoU ($\mathrm{FmIoU}$), the mean intersection over union on the images of the faulty contribution scored against their faulty labels, and by the success rate of a membership inference attack ($\mathrm{MIA}$) against these images. Utility is measured by the retain mIoU ($\mathrm{RmIoU}$) on the training images of the other contributions and the test mIoU ($\mathrm{TmIoU}$) on the held-out half. An unlearned model should match a model retrained without the faulty contribution on all four. For Problem~$(\mathrm{P})$, the residual harm $H[t]$ is the rate at which annotated pedestrians in the held-out half are labeled as road, in excess of that of the retrained model, and $\bar{H}$ is its average over $T$ rounds. The reaction delay counts the rounds from the first symptom until $H[t]$ vanishes. The cost $C(\mathcal{U}[t])$ is the on-board time and uplink volume that the removal adds on the testbed. We compare \method{} with five baselines. FedRecover~\cite{Cao2023SP}, FedOSD~\cite{Pan2025AAAI}, and FedEditor~\cite{Yuan2025TIFS} are told the faulty contribution. FLForensics~\cite{Jia2025NeurIPS} and FLTracer~\cite{Zhang2024TIFS} trace it themselves and hand it to FedRecover for removal.

\subsection{Experimental Results}
\label{subsec:results}

\begin{figure}[t]
\centering
\subfloat[Fleet size.]{\includegraphics[width=0.485\linewidth]{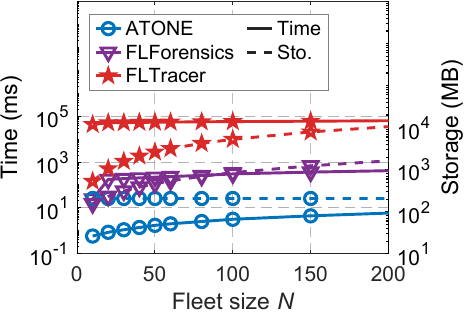}\label{fig:scale}}
\hfill
\subfloat[Subspace dimension.]{\includegraphics[width=0.485\linewidth]{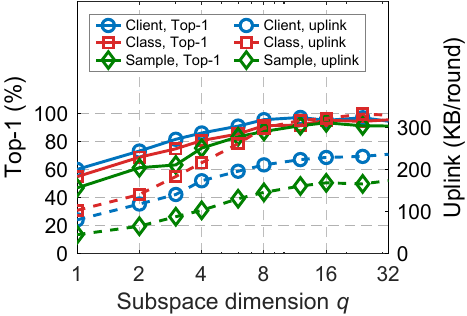}\label{fig:subspace}}
\caption{Scalability of \method{}. (a) Coordinator scoring time and storage as $N$ grows. (b) Top-1 attribution accuracy and uplink per vehicle per round as $q$ grows.}
\label{fig:scalability}
\end{figure}

Table~\ref{tab:main} compares the methods. \method{} stays within 0.36 points of Retrain in $\mathrm{FmIoU}$, $\mathrm{RmIoU}$, and $\mathrm{TmIoU}$, and within 3.3 points in $\mathrm{MIA}$. In eight of nine settings, its $\bar{H}$ is 43\% to 61\% below that of FedEditor, the strongest baseline, although FedEditor is told the faulty contribution. The two tracing baselines have delays of 12.3 to 36.7 rounds against at most 5.7 for \method{}. On the testbed, they also take 2.7 to 7.3 times its time and 1.5 to 4.1 times its uplink per removal.

In Fig.~\ref{fig:harm}, \method{} brings $H[t]$ below 0.1 points by round 47, four rounds before FedEditor and 25 before FLForensics. In Fig.~\ref{fig:attribution}, its Top-1 accuracy is 11.9 to 17.6 points above that of FLForensics. In Fig.~\ref{fig:kappa}, a larger $\kappa$ cuts the uplink up to 6.1 times but raises the delay to at most 11.3 rounds. In Fig.~\ref{fig:ablation}, the first tier alone leaves 1.43 to 2.37 points outside the subspace, and the second alone leaves 1.16 to 1.93 points inside it. Both tiers keep each part below 0.22 points. In Fig.~\ref{fig:scale}, the scoring time of \method{} stays below 6~ms up to $N=200$ and its storage at 206~MB, while FLTracer needs 63.5~s and 10.6~GB. In Fig.~\ref{fig:subspace}, the Top-1 accuracy rises little beyond the default $q=8$, while the uplink per vehicle keeps growing.

\section{Conclusion}
\label{sec:conclusion}

We studied federated unlearning for a vehicle fleet that learns of a faulty contribution only through delayed symptoms. \method{} infers and removes the contribution inside the aggregation rule without storing past checkpoints. On three driving datasets, it yields the lowest residual harm in eight of nine settings and stays within 0.36 points of Retrain.

\bibliographystyle{IEEEtran}
\bibliography{ref}

\end{document}